\documentclass[letterpaper, 10 pt, conference]{ieeeconf}  

\usepackage{booktabs}

\IEEEoverridecommandlockouts                              

\usepackage{setspace}
\usepackage{xcolor}
\usepackage[ruled,vlined, linesnumbered, noend]{algorithm2e}
\usepackage{amsmath}

\usepackage{amsthm}
\usepackage{amsfonts}
\usepackage{cite}
\usepackage{ragged2e}
\usepackage{optidef}
\usepackage{subcaption}
\usepackage{cleveref}
\usepackage{graphicx}
\usepackage{multirow}
\usepackage{url}
\usepackage{diagbox}
\usepackage[font=footnotesize]{caption}
\usepackage{subfig}
\usepackage{soul}
\usepackage{float}
\usepackage{resizegather}
\usepackage{makecell} 
\newtheorem{theorem}{Theorem}

\newtheorem{assumption}{\bf{Assumption}}

\usepackage{tabularx} 
\let\labelindent\relax
\usepackage{enumitem}
\usepackage{import}
\usepackage{amssymb}

\usepackage{tikz}
\usepackage{tkz-euclide}
\usetikzlibrary{shapes.geometric}

\DeclareMathOperator*{\argmin}{argmin} 
\DeclareMathOperator*{\argmax}{argmax}

\title{\bf PRO-SPECT: Probabilistically Safe Scalable Planning for Energy-Aware Coordinated UAV–UGV Teams in Stochastic Environments}
\author{Roger Fowler$^{*}$, Cahit Ikbal Er$^{*}$, Benjamin Johnsenberg, and Yasin Yaz{\i}c{\i}o\u{g}lu
\thanks{$^{*}$These authors contributed equally.}
\thanks{R. Fowler is with the Khoury Department of Computer Sciences at Northeastern University, Boston, MA and a Draper Scholar with The Charles Stark Draper Laboratory, Inc., Cambridge, MA.}
\thanks{C. I. Er is with the Department of Mechanical and Industrial Engineering at Northeastern University, Boston, MA.}
\thanks{B. Johnsenberg is with The Charles Stark Draper Laboratory, Inc., Cambridge, MA.}
\thanks{Y. Yaz{\i}c{\i}o\u{g}lu is with the Departments of Mechanical and Industrial Engineering and Electrical and Computer Engineering at Northeastern University, Boston, MA.}
\thanks {emails: fowler.ro@northeastern.edu, er.c@northeastern.edu, bcjohnson@draper.com, y.yazicioglu@northeastern.edu}
}

\begin{document}
\bstctlcite{bstctl:etal, bstctl:nodash, bstctl:simpurl}
\maketitle
\thispagestyle{empty}
\pagestyle{empty}

\justifying

\begin{abstract}
We consider energy-aware planning for an unmanned aerial vehicle (UAV) and unmanned ground vehicle (UGV) team operating in a stochastic environment. The UAV must visit a set of air points in minimum time while respecting energy constraints, relying on the UGV as a mobile charging station. Unlike prior work that assumed deterministic travel times or used fixed robustness margins, we model travel times as random variables and bound the probability of failure (energy depletion) across the entire mission to a user-specified risk level. We formulate the problem as a Mixed-Integer Program and propose PRO-SPECT, a polynomial-time algorithm that generates risk-bounded plans. The algorithm supports both offline planning and online re-planning, enabling the team to adapt to disturbances while preserving the risk bound. We provide theoretical results on solution feasibility and time complexity. We also demonstrate the performance of our method via numerical comparisons and simulations.
\end{abstract}

\section{Introduction}
Unmanned aerial vehicles (UAVs) are used in various domains such as agriculture, disaster response, and environmental monitoring (e.g., \cite{tokekar2016sensor, boccardo2015uav, manfreda2018use}). However, their limited batteries restrict long-horizon missions. Using unmanned ground vehicles (UGVs) as mobile charging stations addresses this limitation by providing a flexible alternative to building a fixed recharging infrastructure. Leveraging this flexibility requires coordinated planning of the vehicles.

Recently, energy-aware UAV-UGV planning problems have gained significant attention. In \cite{yu2018algorithms}, UAV landing locations on stationary or mobile UGVs are determined via a generalized traveling salesperson problem (GTSP) formulation. In \cite{cai2025energyawareroutingalgorithmmobile}, energy-constrained UAV-UGV routing is addressed with a TSP-guided Monte Carlo Tree Search. In \cite{ropero2019terra}, UAV-UGV trajectories are computed separately via a genetic algorithm and modified A*, with rendezvous locations determined through Voronoi tessellations. In \cite{maini2019coverage}, a two-stage heuristic leverages a road-network-constrained UGV as a refueling station for energy-aware UAV-UGV route planning. In \cite{yu2021rl}, a multi-agent deep reinforcement learning approach guides heterogeneous UAV-UGV teams based on environmental conditions. However, most of these approaches assume complete knowledge of the environment, including obstacle-free environments and known travel times.

In practice, disturbances such as wind, obstacles, or navigation errors can render a plan infeasible. In \cite{thelasingha2024iterative}, an iterative planning framework is proposed to improve solutions online as new information becomes available. In \cite{shi2022risk}, a chance-constrained Markov Decision Process formulation is used to handle stochastic UAV energy consumption with a probabilistic rendezvous policy. In \cite{albarakati2021multiobjective}, multiobjective optimization is applied to the similar problem of Autonomous Underwater Vehicle risk-aware path planning under stochastic water currents. In \cite{Lin2022, er2025rspectrobustscalableplanner}, scalable robust planning algorithms are proposed for persistent surveillance and aerial monitoring with energy-constrained  UAV-UGV teams. 

In this work, unlike prior approaches that rely on deterministic margins (e.g.,~\cite{er2025rspectrobustscalableplanner}), we instead 
optimize under a probabilistic safety requirement encoded as a joint chance constraint. This coupling fundamentally changes the problem structure and requires an approach to ensure global probabilistic safety guarantees. Our contributions are:
\begin{enumerate}
    \item A Mixed-Integer Program formulation for stochastic UAV–UGV energy-aware planning with a joint chance constraint encoding a global safety requirement~\eqref{eq:problem}.
   \item An integrated offline-online polynomial-time planning framework, PRO-SPECT (Alg.\ref{alg:offline_planning}), that computes risk-bounded offline solutions, and adapts to run-time disturbances online via a receding horizon approach.
    \item Formal guarantees on feasibility and worst-case time complexity of PRO-SPECT (Thm.\ref{th:feas_complex}).
    \item Validation via numerical comparisons against well-known optimization methods (Branch and Cut, Simulated Annealing), and related works \cite{ropero2019terra, er2025rspectrobustscalableplanner}; and ROS simulations with wind induced uncertainty.
\end{enumerate}

\section{Problem Formulation}
\label{sec:problem_formulation}
We consider an aerial monitoring mission in a stochastic environment, where a UAV is supported by a UGV (mobile charging station) that transports the UAV through take-off/landing sites. Let $\mathbb{R}$ and $\mathbb{R}_{\geq 0}$ denote the real numbers and non-negative real numbers. Let the environment be:
\begin{equation}
\label{eq:env}
    \mathcal{Q}{:}{=}\{(x,y,z) \in \mathbb{R}^{3} \mid 0 {\leq} x {\leq} \bar{x},\,0 {\leq} y {\leq} \bar{y},\,
 0 {\leq} z {\leq} \bar{z}\},
\end{equation}
containing a feasible space $\mathcal{Q}_f {\subseteq} \mathcal{Q}$. Let $\mathcal{Q}^a_f$ and $\mathcal{Q}^g_f$ denote the feasible points on air and ground, i.e., $\mathcal{Q}^a_f{=}\mathcal{Q}_f {\setminus} \mathcal{Q}^g_f$, $\mathcal{Q}^a_g{=}\mathcal{Q}_f {\setminus} \mathcal{Q}^g_a$.\footnote{Both \( \mathcal{Q}_f \) and \( \mathcal{Q}_f^g \) are connected, i.e., a feasible path exists between any pair of points within these sets.} Let $\ell{:}\mathcal{Q}_f {\times} \mathcal{Q}_f {\rightarrow} \mathbb{R}_{\geq 0}$ be a distance function (shortest path length) between any two points in $\mathcal{Q}_f$. The mission goal is for the UAV-UGV team to visit a set of $n$ aerial points, $\mathcal{P}_{\text{UAV}}{=} \{x_1, \dots, x_n\} {\subset} \mathcal{Q}_f^a$. The team starts from initial position $x_{0} {\in} \mathcal{Q}_f^g$, and must reach the final position $x_{\text{f}} {\in} \mathcal{Q}_f^g$ while ensuring every point in $\mathcal{P}_{\text{UAV}}$ is visited without violation of the UAV's energy constraint, while minimizing mission time.

The UAV cannot continuously fly longer than a specific duration $\bar{\tau}_a$ due to its limited energy.\footnote{For UAVs with near constant power draw, energy depletion can be approximated by flight time.} To this end, we consider an approach where the UGV carries the UAV through a sequence of points, which are \textit{release} and \textit{collect} points. The UAV takes off at a release point, samples a subset of $\mathcal{P}_{\text{UAV}}$, and lands on the UGV at a collect point to be charged and transferred to the next release point. With at most $n{+}2$ points per tour\footnote{Clearly, a solution also cannot be optimal unless the UAV visits at least one point from $\mathcal{P}_{\text{UAV}}$ in every tour (i.e. number of tours $\leq n$).} (release, collect, and $n$ aerial points), we encode the mission plan as:
\begin{equation*}
    {X} = \begin{bmatrix}
        {X}_{1,1} & \cdots & {X}_{1,n+2} \\
        \vdots & \ddots & \vdots \\
        {X}_{n,1} & \cdots & {X}_{n,n+2}
    \end{bmatrix},
\end{equation*}
where  \({X}_{i,j} {\in} \mathcal{Q}_f\) denotes an air point if $ j {\in} \{2, \hdots ,n{+}1\}$, a release point if $j{=}1$, and a collect point if $j{=}n{+}2$. Each row of $X$ (denoted as $X_{i}$) encodes a tour that involves the UAV traveling through the sequence of waypoints in that row via shortest paths and meeting with the UGV at the collect point. Thus, any plan with fewer than $n$ tours can be encoded with the remaining rows as ${X}_{i,1}{=}{X}_{i,2}{=}\hdots{=}{X}_{i,n{+}2}$ i.e., zero length, trivial tours. After each tour is executed, the UGV transfers the UAV from ${X}_{i,n+2}$ (collect point) to ${X}_{i+1,1}$ (next release point), or $x_{\text{f}}$ if $i{=}n$, while recharging the UAV. For the team, the total time to execute the plan $X$ is
\begin{equation}
\small
\begin{aligned}
\tau(X) &= \tau_g(x_{0}, X_{1,1}) + \tau_g(X_{n,n+2}, x_{\text{f}}) \\
&\quad + \sum_{i=1}^{n} \max(\tau_a(X_{i}), \tau_g(X_{i,1}, X_{i,n+2})) \\
&\quad + \sum_{i=1}^{n-1} \max(\tau_g(X_{i,n+2}, X_{i+1,1}), \tau_c(X_{i})),
\end{aligned}
\label{eq:objective_function}
\end{equation}
where $\tau_a$ and $\tau_g$ denote UAV tour time and UGV travel time respectively, and $\tau_a {=} \tau_g {=} 0$ for zero-length paths. Here, $\tau_a(X_i)$ is the UAV time for tour $i$ from release point $X_{i,1}$ through points $X_{i,2}, \ldots, X_{i,n+1}$ to collect point $X_{i,n+2}$. The terms $\tau_g(x_{0}, X_{1,1})$ and $\tau_g(X_{n,n+2}, x_{\text{f}})$ are the UGV times from start to first release and from final collect to end, respectively. The term $\tau_g(X_{i,n+2}, X_{i+1,1})$ is the UGV travel time between consecutive tours. The recharge time after tour $i$ is $\tau_c(X_i) {\geq} 0$, ensuring the UAV begins each tour fully charged.\footnote{Battery swapping can be modeled by setting $\tau_c(X_i)$ to a small constant.}

Different from earlier works such as \cite{er2025rspectrobustscalableplanner} that consider deterministic travel times (energy consumption), we model stochastic travel times and impose probabilistic guarantees on safety, i.e., the UAV never has to fly longer than $\bar{\tau}_a$. More specifically, we consider the problem of minimizing the expected time to complete the mission while achieving a desired bound on the probability of failure as in \eqref{eq:risk} where $p_r$ is a user-specified acceptable risk level, i.e.,
\begin{subequations}\label{eq:problem}
\begin{align}
\min_X \quad & \mathbb{E}[\tau(X)] \label{eq:objective},\\
\text{s.t.} \quad & \forall x_k \in \mathcal{P}_{\text{UAV}}, \quad \exists X_{i,j} = x_k, \label{eq:coverage}\\
& P\left(\tau_a(X_i), \tau_g(X_{i,1}, X_{i,n+2}) \leq \bar{\tau}_a, \; \forall i\right) \geq 1 - p_r, \label{eq:risk}\\
& X_{i,1}, X_{i,n+2} \in \mathcal{Q}_f^g, \quad \forall i, \label{eq:ground}\\
& X_{i,j} \in \mathcal{P}_{\text{UAV}} \cup \{X_{i,j-1}\}, \quad \forall j \in \{2,\ldots,n+1\}, \label{eq:air}
\end{align}
\end{subequations}
where \eqref{eq:objective} minimizes expected mission completion time, \eqref{eq:coverage} ensures all points in $\mathcal{P}_{\text{UAV}}$ are visited, \eqref{eq:risk} bounds the probability of failure across all tours by user-specified $p_r$,\footnote{If no feasible solution exists when each air point comprises its own tour, then $\bar{\tau}_a$ is insufficient to reach some air points, and we assume the problem admits no feasible solution.} \eqref{eq:ground} constrains release/collect points to the feasible ground, and \eqref{eq:air} allows only release/collect points or their immediate repetitions to not be members of $\mathcal{P}_{\text{UAV}}$.

\section{Proposed Solution}
\label{sec:proposed_solution}

\subsection{Environmental Model}
To model uncertainty, we define mean and variance functions to estimate travel times. We consider only the first two cumulants, as higher-order moments are typically unavailable or unreliable in practical environmental models. Let $t_a(u,v) \sim T_a(u,v)$ denote the UAV travel time between two points, analogous to $\tau_g$. We define the UAV environmental model as:
\begin{equation}
M_{\theta,a}(u, v) \equiv \mathbb{E}[t_a(u, v)], \quad
V_{\theta,a}(u, v) \equiv \text{Var}[t_a(u, v)],
\label{eq:uav_model}
\end{equation}
where $M_{\theta,a}(u, v)$ and $V_{\theta,a}(u, v)$ denote the expected travel time and its variance for the UAV between points $u$ and $v$, respectively. For the UGV:
\begin{equation}
M_{\theta,g}(u, v) \equiv \mathbb{E}[\tau_g(u, v)], \quad
V_{\theta,g}(u, v) \equiv \text{Var}[\tau_g(u, v)],
\label{eq:ugv_model}
\end{equation}
where $M_{\theta,g}(u, v)$ and $V_{\theta,g}(u, v)$ denote the expected travel time and its variance for the UGV between points $u$ and $v$, respectively. The parameter vector $\theta$ captures environmental variables such as wind, obstacles, and navigation uncertainty. Note that the approach is otherwise agnostic to the underlying noise model and user preferences. 

\begin{assumption}
\label{as:edges}
All travel time distributions are bounded, and travel times between
different pairs of points, both in the air and on the ground, are mutually
independent.
\end{assumption}

From Assumption~\ref{as:edges}, $M_{\theta,*}$ and $V_{\theta,*}$ depend only on endpoints, allowing the functions to be evaluated and cached per \textit{pair of points (edges).} This enables the total tour statistics to decompose as sums over edges. For tour $i$ with waypoints $X_{i,1}, \hdots, X_{i,n+2}$:
\begin{equation}
\begin{aligned}
\hat{\mu}_{a,i} &= \sum_{j=1}^{n+1} M_{\theta,a}(X_{i,j}, X_{i,j+1}), &
\hat{\sigma}^2_{a,i} &= \sum_{j=1}^{n+1} V_{\theta,a}(X_{i,j}, X_{i,j+1}), \\
\hat{\mu}_{g,i} &= M_{\theta,g}(X_{i,1}, X_{i,n+2}), &
\hat{\sigma}^2_{g,i} &= V_{\theta,g}(X_{i,1}, X_{i,n+2}).
\end{aligned}
\label{eq:time_propagation}
\end{equation}

\subsection{Probabilistic Constraint Evaluation}
\label{subsec:constraint_eval}
We evaluate the risk constraint through Gaussian surrogates, without assuming the
true travel time distributions are Gaussian. For each tour $i$, let $f_{a,i}$ and $f_{g,i}$ denote the distributions of the respective UAV and UGV travel times. Due to Assumption~\ref{as:edges}, $f_{a,i}$ and $f_{g,i}$ both have bounded support. For each of them, we can define a Gaussian surrogate with the same mean and variance, i.e.,  $\varphi_{a,i}=\mathcal{N}(\hat{\mu}_{a,i},\hat{\sigma}^2_{a,i})$ and $\varphi_{g,i}=\mathcal{N}(\hat{\mu}_{g,i},\hat{\sigma}^2_{g,i})$. Since every $f_{a,i}$ and $f_{g,i}$ have bounded support, there exist some finite $\bar{t}_{a,i}$ and $\bar{t}_{g,i}$ beyond which the upper tail of each original distribution is dominated by that of its Gaussian surrogate, i.e.,
\begin{align}
\label{eq:tail_domination}
\int_{t_{*,i}}^{\infty} f_{*,i}(s)\,ds &\leq \int_{t_{*,i}}^{\infty} \varphi_{*,i}(s)\,ds, \quad \forall t_{*,i} \geq \bar{t}_{*,i}, 
\end{align}

For any $p\in[0,1]$, let $\Theta_{a,i}(p)$ and $\Theta_{g,i}(p)$ denote points beyond which the respective Gaussian functions have mass equal to $p$, i.e., 
\begin{equation}
\int_{\Theta_{g,i}(p)}^{\infty} \varphi_{g,i}(s)\,ds = \int_{\Theta_{a,i}(p)}^{\infty} \varphi_{a,i}(s)\,ds = p.
\label{eq:quantile_def}
\end{equation}
Note that all $\Theta_{g,i}(p)$ and $\Theta_{a,i}(p)$ are monotonically decreasing functions, i.e., as $p$ increases the point becomes smaller. Now, let $\bar{p}_r \in [0,1]$ be defined as: 
\begin{equation}
\bar{p}_r= \max\{p \in [0,1] \mid \Theta_{g,i}(p)\geq \bar{t}_{g,i}, \Theta_{a,i}(p)\geq \bar{t}_{a,i}, \forall i\}.
\label{eq:prbar}
\end{equation}

For any $p \leq \bar{p}_r$, the corresponding surrogate tails dominate the true tails, i.e., 
\begin{align}
\label{eq:tail_domination3}
\int_{\Theta_{*,i}(p)}^{\infty} f_{*,i}(s)\,ds &\leq \int_{\Theta_{*,i}(p)}^{\infty} \varphi_{*,i}(s)\,ds, 
\end{align}

Accordingly, satisfaction of the risk bound~\eqref{eq:risk} under the surrogate implies
satisfaction under the true distribution:
\begin{subequations}
\label{eq:survival_true}
\begin{align}
P\!\left(\tau_a(X_i) \leq \bar{\tau}_a\right) &= 1 - \int_{\bar{\tau}_a}^{\infty} f_{a,i}(s)\,ds, \label{eq:survival_true_a}\\
P\!\left(\tau_g(X_{i,1}, X_{i,n+2}) \leq \bar{\tau}_a\right) &= 1 - \int_{\bar{\tau}_a}^{\infty} f_{g,i}(s)\,ds. \label{eq:survival_true_g}
\end{align}
\end{subequations}

Accordingly, for any $p \leq \bar{p}_r$
\begin{align}
\label{eq:conservatism}
\int_{\bar{\tau}_a}^{\infty} \varphi_{*,i}(s)\,ds \leq p
&\;\implies\;
\int_{\bar{\tau}_a}^{\infty} f_{*,i}(s)\,ds \leq p, 
\end{align}

By \eqref{eq:conservatism}, it suffices to control the surrogate tails:
whenever the Gaussian mass beyond $\bar{\tau}_a$ is at most $p$, the true
mass beyond $\bar{\tau}_a$ is at most $p$ as well.

Using \eqref{eq:survival_true} and \eqref{eq:conservatism}, for any $p \leq \bar{p}_r$,
\begin{subequations}
\label{eq:surrogate_to_true}
\begin{small}
\begin{align}
\int_{\bar{\tau}_a}^{\infty} \varphi_{a,i}(s)\,ds \leq p
&\implies
P\!\left(\tau_a(X_i) \leq \bar{\tau}_a\right) \geq 1 - p,
\label{eq:surrogate_to_true_a}\\
\int_{\bar{\tau}_a}^{\infty} \varphi_{g,i}(s)\,ds \leq p
&\implies
P\!\left(\tau_g(X_{i,1},X_{i,n+2}) \leq \bar{\tau}_a\right) \geq 1 - p.
\label{eq:surrogate_to_true_g}
\end{align}
\end{small}
\end{subequations}

The inequalities on the left side of \eqref{eq:surrogate_to_true} can be checked easily since the surrogate tail is available in closed form
\begin{equation}
    \int_{\bar{\tau}_a}^{\infty}\varphi_{*,i}(s)\,ds
= 1 - \Phi \left(\frac{\bar{\tau}_a - \hat{\mu}_{*,i}}{\hat{\sigma}_{*,i}}\right),
\end{equation}
where $\Phi$ denotes the cumulative density function of the Gaussian:
\begin{equation}
\Phi\!\left(\frac{\bar{\tau}_a - \hat{\mu}_{*,i}}{\hat{\sigma}_{*,i}}\right)
= \frac{1}{2}\!\left(1 + \operatorname{erf}\!\left(\frac{\bar{\tau}_a - \hat{\mu}_{*,i}}{\hat{\sigma}_{*,i}\sqrt{2}}\right)\right),
\label{eq:gaussian_cdf}
\end{equation}
where $\text{erf}$ is the error function.
We define $p_{s,a}$ and $p_{s,g}$ as the probability of each agent's respective Gaussian surrogate not exceeding $\bar{\tau}_a$ during tour $i$:
\begin{align}
p_{s,a}(X_i) &= \Phi\!\left(\frac{\bar{\tau}_a - \hat{\mu}_{a,i}}{\sqrt{\hat{\sigma}^2_{a,i}}}\right) = 1 - \int_{\bar{\tau}_a}^{\infty} \varphi_{a,i}(s)\,ds, \label{eq:survival_uav}\\
p_{s,g}(X_{i,1}, X_{i,n+2}) &= \Phi\!\left(\frac{\bar{\tau}_a - \hat{\mu}_{g,i}}{\sqrt{\hat{\sigma}^2_{g,i}}}\right) = 1 - \int_{\bar{\tau}_a}^{\infty} \varphi_{g,i}(s)\,ds. \label{eq:survival_ugv}
\end{align}
Algorithm \ref{alg:offline_planning} will use these surrogate success probabilities to design feasible tours.

\subsection{Tour Construction}
The key challenge of \eqref{eq:problem} is the constraint \eqref{eq:risk}. Unlike \cite{er2025rspectrobustscalableplanner}, where constraints could be applied per tour, all tours are now coupled through a joint probability bound. As a result, tour boundaries and release/collect points must be selected jointly across the mission, rather than per tour. With the mutual independence of travel times from Assumption~\ref{as:edges}, the mission success probability based on Gaussian surrogates equals $\prod_{i=1}^m p_{s,i}$. Our approach is to enforce \eqref{eq:risk} by constructing tours satisfying
\begin{equation}
    \prod_{i=1}^m p_{s,i} \geq 1 - p_r,
    \label{eq:mission_failure_risk}
\end{equation}
where $p_{s,i}=p_{s,a}(X_i){\cdot}p_{s,g}(X_{i,1}, X_{i,n+2})$ is the surrogate tour success probability and $m$ is the number of tours. To make the joint optimization more tractable, we decouple the problem by fixing the visit order. To this end, we solve a TSP on $\mathcal{P}_{\text{UAV}}$ using edge costs from $M_{\theta,a}$ to obtain an ordered sequence $\mathcal{S}$.\footnote{If $M_{\theta,*}(u, v) {=} M_{\theta,*}(v, u)$ and $V_{\theta,*}(u, v) {=} V_{\theta,*}(v, u)$ the environmental model is \textit{symmetric}. If the environmental model is not \textit{symmetric}, we consider an asymmetric TSP (ATSP) without loss of generality.} Subsequent optimization focuses on selecting tour boundaries and release/collect points. We represent tour partitioning by a set of cut indices $K {=} \{k_1,\ldots,k_{m-1}\}$, partitioning $\mathcal{S}$ into tours:
\begin{equation}
    \mathcal{T}_{i} = \begin{cases}
        \mathcal{S}_{1 : k_{i} - 1} & \textbf{if } i = 1 \\
        \mathcal{S}_{k_{i-1} : n} & \textbf{if } i = m \\
        \mathcal{S}_{k_{i-1} : k_{i} - 1 } & \textbf{otherwise}
    \end{cases}
\end{equation}

Each tour (row) includes a release point $r_i$ (visited first), a collect point $c_i$ (visited last), and their ground projections, with $c_i$ repeated such that $|X_{i,:}| = n+2$:
\begin{equation}
    X_{i,:} = ( \pi(r_i), r_i, \mathcal{T}_{i} \setminus \{r_i,c_i\}, c_i^{\times(n-|\mathcal{T}_i|+1)}, \pi(c_i) ),
\end{equation}
where $\pi : Q_f^a \rightarrow Q_f^g$ denotes a projection function that maps an aerial point to a feasible ground point. Specifically, $\pi(u)$ returns the closest feasible ground point to $u$, defined as $\pi(u) = \argmin_{v \in Q_f^g} \, \ell(u,v).$ From independence (Assumption~\ref{as:edges}) surrogate tour success probability combines UAV and UGV success:
\begin{equation}
    p_{s,i} = p_{s,a}[(\pi(r_i), r_i, \mathcal{T}_{i} \setminus \{r_i,c_i\}, c_i, \pi(c_i))] \cdot p_{s,g}[(\pi(r_i), \pi(c_i))].
    \label{eq:tour_prob_success}
\end{equation}

We then pose the tour construction problem by maximizing the product of probabilities, equivalent to maximizing the sum of log-probabilities since $p_{s,i} \in [0,1]$:
\begin{equation}
    \max_{K,R,C} \sum_{i=1}^m \log(p_{s,i}),
    \label{eq:prob_success}
\end{equation}
where $R{{=}\{r_1,{\dots},r_m\}}$ and  $C{{=}\{c_1,{\dots},c_m\}}$ are the sets of release and collect points, respectively. Then, \eqref{eq:prob_success} is solved for the minimum $m$ satisfying \eqref{eq:mission_failure_risk}.

\subsection{Dynamic Program}
We solve \eqref{eq:prob_success} with dynamic programming, building up the tour partition incrementally: an $m$-tour solution is constructed by extending an $(m{-}1)$-tour solution with one additional tour, sweeping over all feasible cut indices and release/collect point assignments. 
We define $D_{j,m}$ as the maximum log-probability achievable by partitioning $S_{1:j}$ into $m$ tours. The function $f(a,b,r,c)$ returns the log-probability contribution of a single tour segment $\mathcal{S}_{a:b}$ with release index $r$ and collect index $c$:
\begin{equation}
    \begin{split}
f(a,b,r,c) &= \log(p_{s,a}[(\pi(\mathcal{S}_r), \mathcal{S}_r, \mathcal{S}_{(a:b) \setminus \{r,c\}}, \mathcal{S}_c, \pi(\mathcal{S}_c))]) \\
&\quad+ \log(p_{s,g}[(\pi(\mathcal{S}_r), \pi(\mathcal{S}_c))]).
    \end{split}
    \label{eq:log_prob_success}
\end{equation}

Using the definition of $D_{j,m}$ and the segment cost function $f$, the recurrence relation is:
\begin{equation}
    D_{j,m} = \left\{\begin{array}{clc}
        \displaystyle\max_{r_m, c_m \in \{1:j\}} & f(1,j,r_m,c_m) & \textbf{if   } m = 1 \\[10pt]
        \displaystyle\max_{\substack{k_{m-1} \in \{m+1:j-1\}\\ r_m, c_m \in \{k_{m-1}:j\}}} &
        \begin{split}
            & D_{k_{m-1}-1, m-1} \\ & + f(k_{m-1},j,r_m,c_m) 
        \end{split}
         & \textbf{if   } m > 1
    \end{array}\right.
    \label{eq:dynamic_program}
\end{equation}

We sweep $m$ from 1 to $n$ until $D_{n,m} {\geq} \log(1{-}p_r)$, storing optimal $k_{m-1}, r_m, c_m$ for each $(j,m)$ to reconstruct $K,R,C$.

When the given travel time model is symmetric, i.e., $f(a,b,r,c) {=} f(a,b,c,r)$, we choose $r {\leq} c$ to ensure the release point does not succeed the collect point in $\mathcal{S}$, respecting the TSP ordering and likely minimizing UGV inter-tour time.

\subsection{Offline Planning}
\label{sec:offline}
Offline planning considers problem \eqref{eq:problem} in full. Some algorithmic inputs are notated with tildes as they may differ during online planning. For offline planning, inputs follow as: UAV points $\tilde{\mathcal{P}}_\text{UAV} {=} \mathcal{P}_\text{UAV}$, initial positions $\tilde{x}_{0,a} {=} \tilde{x}_{0,g} {=} x_0$, flight time $\tilde{\tau}_0 {=} 0$, final position $\tilde{x}_f {=} x_f$, and allowable risk $\tilde{p}_r {=} p_r$. Output mission plan $\tilde{X}$ is $n\times n+2$ as required.

\subsection{Online Re-planning}
\label{sec:online}
Offline planning assumes agents start co-located with full battery. Online re-planning relaxes these assumptions with a receding horizon approach, re-solving a subset of the offline problem at tour boundaries with an adjusted risk budget to preserve the global guarantee \eqref{eq:mission_failure_risk}.

\subsubsection{Re-planning Between Tours}
When the UAV and UGV are together before tour $i$, we have current UAV position $\tilde{x}_{0,a}$ and UGV position $\tilde{x}_{0,g}$ satisfying $\tilde{x}_{0,a} = \tilde{x}_{0,g}$, but the UAV can have a residual charge deficit $\tilde{\tau}_0 {\geq} 0$ from the previous tour, modifying the recharge time before the next tour. To re-plan the next $\tilde{m}$ tours:

\begin{enumerate}[label=\roman*)]
    \item Select planning horizon: $\tilde{m} \in \{0:m-i\}$
    \item Solve with modified inputs:
    \begin{itemize}
        \item $\tilde{\mathcal{P}}_{\text{UAV}} = \bigcup_{k=i}^{i+\tilde{m}} \{X_{k,2}, \ldots, X_{k,n+1}\}$
        \item $\tilde{x}_{0,a} = \tilde{x}_{0,g} = X_{i,2}, \quad \tilde{x}_f = \pi(X_{i+\tilde{m}, n+1})$
        \item $\tilde{p}_r = 1 - \exp\left( \log(1-p_r) - \sum_{j=i+\tilde{m}+1}^{m} \log p_{s,j} \right)$
    \end{itemize}
    \item Update plan: $X_{i:i+\tilde{m}, :} = \tilde{X}$ 
\end{enumerate}

Since the risk budget $\tilde{p}_r$ is derived from \eqref{eq:mission_failure_risk} with the success probabilities of tours beyond the horizon, the bound is preserved regardless of when re-planning occurs.

\subsubsection{Re-planning During a Tour}

Re-planning can occur at any time during execution. When the UAV and UGV are not together during tour $i$, we have current UAV position $\tilde{x}_{0,a} {\neq} \tilde{x}_{0,g}$ with elapsed flight time $\tilde{\tau}_0 {\geq} 0$, reducing the remaining energy budget for the current tour. $\tilde{m}$, $\tilde{p}_r$, and $\tilde{x}_f$ are the same as above. If $\tilde j$ is the next index in the tour, we have $\tilde{\mathcal{P}}_{\text{UAV}} = \{X_{k,\tilde j}, \ldots, X_{k,n+1}\} \cup \bigcup_{k=i+1}^{i+\tilde{m}} \{X_{k,2}, \ldots, X_{k,n+1}\}$.

We modify the DP to handle the partial tour by replacing $f$ with $f'$, where the release point has already passed:
\begin{equation}
    \begin{split}
    f'(1,b,\cdot,c) &= \log(p'_{s,a}[(\tilde{x}_{0,a}, \mathcal{S}_{(1:b) \setminus \{c\}}, \mathcal{S}_c, \pi(\mathcal{S}_c))]) \\
    &\quad+ \log(p'_{s,g}[(\tilde{x}_{0,g}, \pi(\mathcal{S}_c))]),
    \end{split}
    \label{eq:log_prob_success_online}
\end{equation}
and update success probabilities to account for elapsed time:
\begin{equation}
    p'_{s,a} = \Phi \left(\frac{\bar{\tau}_a - \tilde{\tau}_0 - \hat{\mu}_{a,1}}{\sqrt{\hat{\sigma}^2_{a,1}}}\right), \quad
    p'_{s,g} = \Phi \left(\frac{\bar{\tau}_a - \tilde{\tau}_0 - \hat{\mu}_{g,1}}{\sqrt{\hat{\sigma}^2_{g,1}}}\right),
    \label{eq:energy_constraints_online}
\end{equation}
where $(\hat{\mu}_{*,1}, \hat{\sigma}^2_{*,1})$ are summed only over remaining points.

\subsection{Proposed Algorithm}
We propose PRO-SPECT (Alg.~\ref{alg:offline_planning}), which returns a feasible solution to \eqref{eq:problem} with a polynomial-time worst-case complexity. A visualization of the algorithm is provided in Fig.\ref{fig:algorithm_steps}.

\SetKwComment{Comment}{//}{}

\begin{algorithm}[h!]
\renewcommand{\AlCapSty}[1]{\normalfont\footnotesize{\textbf{#1}}\unskip}\footnotesize
\SetAlgoLined
\SetAlgoNoEnd
\DontPrintSemicolon
\SetKwInOut{Input}{Input}
\SetKwInOut{Output}{Output}

\Input{
    $\tilde{\mathcal{P}}_{\text{UAV}}$ (Set of $\tilde{n}$ points to visit), 
    $n$ (Size of full $\mathcal{P}_\text{UAV}$),
    $\mathcal{Q}_f^g$ (Feasible Ground),
    $\mathcal{Q}_f^a$ (Feasible Air), 
    $\tilde{x}_{0,a}$ (UAV start position), $\tilde{\tau}_0$ (UAV initial flight time), 
    $\tilde{x}_{0,g}$ (UGV start position), $\tilde{x}_\text{f}$ (Final position), 
    $\bar{\tau}_a$ (Maximum flight time), 
    $\tilde{p}_r$ (User-specified risk parameter),
    $M_{\theta,a},V_{\theta,a}$, 
    $C_{\theta,a}$,
    $M_{\theta,g},V_{\theta,g}$ (UAV and UGV travel time functions)
}
\Output{
    ${\tilde{X}}$ (Mission Plan)
}
\SetKw{continue}{continue}

\caption{\scriptsize Probabilistically Safe UAV-UGV Planning (PRO-SPECT)}
\label{alg:offline_planning}

\textbf{Step 1: Finding an approximate solution $\tilde{\mathcal{S}}$ to TSP on $\tilde{\mathcal{P}}_{\text{UAV}}$}

Use an approximate TSP solver \cite{christofides2022worst}\footnotemark on $\tilde{\mathcal{P}}_{\text{UAV}}$ to find $\tilde{\mathcal{S}}$ such that the edge cost is calculated with $M_{\theta,a}$ and \;
$\tilde{\mathcal{S}}_1 = \underset{x \in \mathcal{P}_{\text{UAV}}}{\argmin} \, M_{\theta,a}(\tilde{x}_{0,a}, x), \quad \tilde{\mathcal{S}}_{\tilde{n}} = \underset{x \in \mathcal{P}_{\text{UAV}}}{\argmin} \, M_{\theta,a}(x, \tilde{x}_\text{f})$

\textbf{Step 2: Creating feasible UAV tours from $\tilde{\mathcal{S}}$}

\textbf{Initialize:} $\tilde{X}, D, K, R, C \leftarrow \emptyset$, \textsc{ValidPlanFlag} $\leftarrow$ False\;
$\tilde{n} \leftarrow |\tilde{\mathcal{P}}_{\text{UAV}}|$\;

\For{$m \leftarrow 1$ \KwTo $\tilde{n}$} {
    \For{$j \leftarrow m$ \KwTo $\tilde{n}$} {
        \If{$\tilde{x}_{0,a} \neq \tilde{x}_{0,g}$ \textbf{and} $m=1$} {
            $\hat c \leftarrow \underset{c \in \{1 : j\}}{\argmax} \quad f'(1, j, \cdot, c)$\;
            $D_{j,m} \leftarrow f'(1, j, \cdot, \hat c)$\;
            $K_{j,m} \leftarrow -1$, $R_{j,m} \leftarrow -1$, $C_{j,m} \leftarrow \hat c$\;
        } \ElseIf {$m=1$} {
            $\hat r, \hat c \leftarrow \underset{r, c \in \{1 : j\}}{\argmax} \quad f(1, j, r, c)$\;
            $D_{j,m} \leftarrow f(1, j, \hat r, \hat c)$\;
            $K_{j,m} \leftarrow -1$, $R_{j,m} \leftarrow \hat r$, $C_{j,m} \leftarrow \hat c$\;
        } \Else {
            $\hat k, \hat r, \hat c \leftarrow \underset{\substack{k \in \{m+1:j-1\} \\ r, c \in \{1:j\}}}{\argmax} \quad D_{k-1, m-1} + f(k, j,r,c)$\;
            $D_{j,m} \leftarrow D_{\hat k-1, m-1} + f(\hat{k}, j, \hat r, \hat c)$\;
            $K_{j,m} \leftarrow \hat k$, $R_{j,m} \leftarrow \hat r$, $C_{j,m} \leftarrow \hat c$\;
        }
    }
    \If{$D_{\tilde{n},m} \geq \log(1 - p_r)$} {
        \textsc{ValidPlanFlag} $\leftarrow$ True\;
        \textbf{break}\;
    }
}

\If{\textbf{not} \textsc{ValidPlanFlag}} {
    \Return{$\emptyset$} \;
}
$\hat n \leftarrow \tilde{n}$, $\hat m \leftarrow m$ \;
\While{$\hat m > 1$} {
    $k_{\hat m} \leftarrow K_{\hat n, \hat m}$, $r_{\hat m} \leftarrow R_{\hat n, \hat m}$, $c_{\hat m} \leftarrow C_{\hat n, \hat m}$\;
    $\tilde{X}_{\hat{m},:} \leftarrow (\,\pi( \tilde{\mathcal{S}}_{r_{\hat m}} ),\, \tilde{\mathcal{S}}_{r_{\hat m}}, \, \tilde{\mathcal{S}}_{(k_{\hat m} : \hat n) \setminus \{r_{\hat m},c_{\hat m}\}},$
    $\tilde{\mathcal{S}}_{c_{\hat m}}^{\times (n - \hat n + k_{\hat m} - 1)}, \, \pi(\tilde{\mathcal{S}}_{c_{\hat m}}) \,)$\;
    $\hat n \leftarrow k_{\hat m} - 1$, $\hat m \leftarrow \hat m - 1$ \;
}
$r_1 \leftarrow R_{\hat n, 1}$, $c_1 \leftarrow C_{\hat n, 1}$\;
\If{$x_{0,a} \neq x_{0,g}$} {
    $\tilde{X}_{1,:} \leftarrow ( \,x_{0,g}, x_{0,a}, \, \tilde{\mathcal{S}}_{(1 : \hat n) \setminus \{c_1\}}, \, \tilde{\mathcal{S}}_{c_1}^{\times (n - \hat n)},\, \pi(\tilde{\mathcal{S}}_{c_1}) \,)$\;
} \Else {
    $\tilde{X}_{1,:} \leftarrow (\, \pi(\tilde{\mathcal{S}}_{r_1}),\, \tilde{\mathcal{S}}_{r_1}, \, \tilde{\mathcal{S}}_{(1 : \hat n) \setminus \{r_1,c_1\}}, \, \tilde{\mathcal{S}}_{c_1}^{\times (n - \hat n)},\, \pi(\tilde{\mathcal{S}}_{c_1}) \,)$\;
}
$\tilde{X}_{m+1:n,:} \leftarrow \tilde{x}_\text{f}$\;
\Return{$\tilde{X}$}
\end{algorithm}
\footnotetext{We can use any other TSP solver in these steps. However, this might change the worst-case time complexity of Alg.~\ref{alg:offline_planning}.}

\begin{figure}[htbp]
\centering
\begin{subfigure}[t]{0.3\linewidth}
    \includegraphics[width=\textwidth]{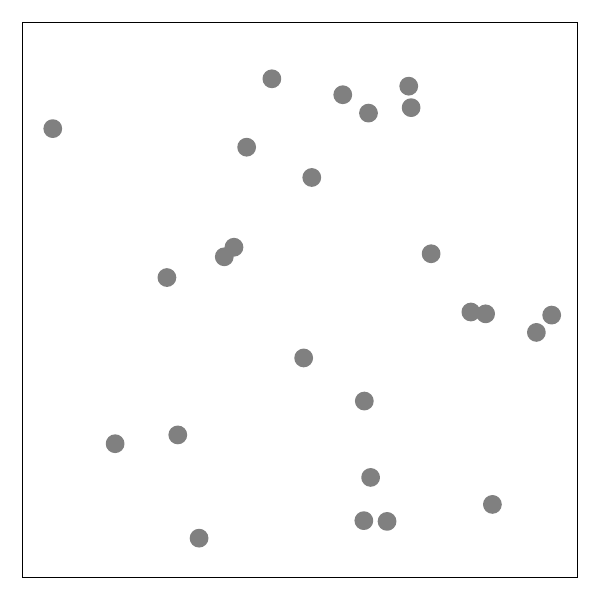}
    \caption{$\mathcal{P}_\text{UAV}$}
    \label{fig:step0}
\end{subfigure}
\hfill
\begin{subfigure}[t]{0.3\linewidth}
    \includegraphics[width=\textwidth]{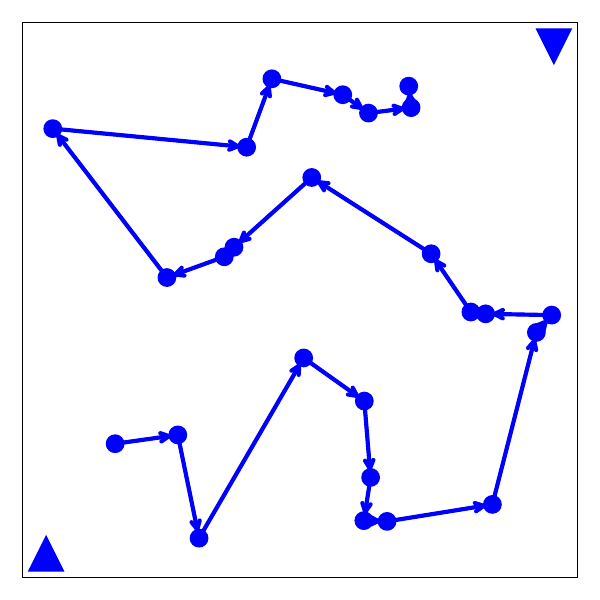}
    \caption{Step 1}
    \label{fig:step1}
\end{subfigure}
\hfill
\begin{subfigure}[t]{0.3\linewidth}
    \includegraphics[width=\textwidth]{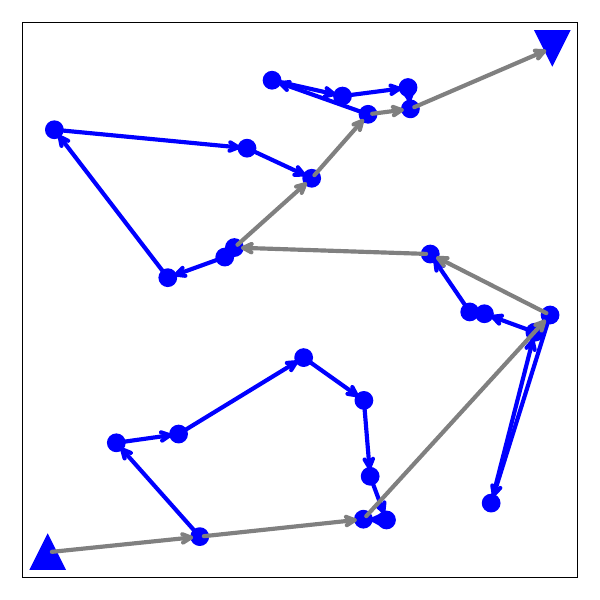}
    \caption{Step 2}
    \label{fig:step2}
\end{subfigure}
\caption{Offline plan generated by PRO-SPECT : (a) $\mathcal{P}_{\text{UAV}}$; (b) TSP solution for $\mathcal{S}$ with $x_{0}$ (lower left) and $x_\text{f}$ (upper right);  (c) tour construction via dynamic programming with UGV path (gray)  and UAV tours (blue).}
\label{fig:algorithm_steps}
\end{figure}

In Step~1, an approximate TSP is solved on $\tilde{\mathcal{P}}_{\text{UAV}}$ using $M_{\theta,a}$ as edge costs, producing the sequence $\tilde{\mathcal{S}}$ (lines 1-3). Here, $\tilde{\mathcal{S}}_1$ is set to the point closest to the UAV start position, and $\tilde{\mathcal{S}}_{\tilde{n}}$ to the point closest to the final position. 

In Step~2, the tours are constructed. First, $\tilde{X}$ (plan), $K$ (cut indices storing the optimal $k_{m-1}$ for each segment), $R$ (release indices storing the optimal release point $r$ for each segment), $C$ (collect indices storing the optimal collect point $c$ for each segment), and $D$ (DP table storing the maximum log-probability achievable for each $(j,m)$) are initiated (line 5). Then, the number of tours $m$ is swept from $1$ to $\tilde{n}$ (line 7), and for each $m$, the endpoint index $j$ is swept from $m$ to $\tilde{n}$ (line 8). Three cases are distinguished: (i) if the team is separated and $m{=}1$ (lines 9-12), then the partial first tour is handled using $f'$, optimizing only the collect point; (ii)~if $m{=}1$ and the team is co-located (lines 13-16), then both release and collect points are optimized using $f$; (iii) if $m{>}1$ (lines 17-20), then the cut, release, and collect are jointly optimized for the last segment and combined with the best prior solution. The sweep terminates once $D_{\tilde{n},m}$ (the maximum log-probability achievable by partitioning all $\tilde{n}$ points into $m$ tours) exceeds $\log(1{-}p_r)$ (lines 21-23).

The backtracking phase (lines 26-36) traces through the stored tables to create the tours. Each tour is reconstructed by retrieving the optimal cut, release, and collect indices and assembling the corresponding waypoint sequence (lines 28-30). The notation $\tilde{\mathcal{S}}_c^{\times(n-\hat{n})}$ denotes the repetition of point $\tilde{\mathcal{S}}_c$ exactly $(n{-}\hat{n})$ times, padding the remaining points in the row. The first tour is then constructed from the current robot positions if they are separated (line 33), or via standard projected release otherwise (line 35), with padding applied analogously in both cases. Finally, the rows after the last tour are set to $\tilde{x}_f$ (line 36), and $\tilde{X}$ is returned.

We provide formal guarantees on the computational complexity of Alg.\ref{alg:offline_planning} and the feasibility of the returned solution. Alg.\ref{alg:offline_planning} returns a plan visiting each point in $\tilde{\mathcal{P}}_\text{UAV} {\subseteq} \mathcal{P}_\text{UAV}$, with a worst case time complexity that is polynomial with the number of points in $\tilde{\mathcal{P}}_\text{UAV}$.

\begin{theorem}
\label{th:feas_complex}
Under Assumption~\ref{as:edges}, for any $p_r \leq \bar{p}_r$ with $\bar{p}_r$ given in \eqref{eq:prbar}, Alg.~\ref{alg:offline_planning} returns a feasible solution to
\eqref{eq:problem}, with worst-case time complexity $O(\tilde{n}^5)$.
\end{theorem}
\begin{proof}
\textit{Feasibility:} Step 1 solves a TSP, so every point in
$\tilde{\mathcal{P}}_{\text{UAV}}$ appears in $\tilde{\mathcal{S}}$, satisfying
\eqref{eq:coverage}. 
Step 2 terminates only when $D_{\tilde{n},m} \ge \log(1-p_r)$, i.e., $\prod_i p_{s,i} \ge 1-p_r$ where $p_{s,i}$ is the surrogate tour success probability. Each factor lies in $[0,1]$, so $\prod_i p_{s,i} = \prod_i p_{s,a} {\cdot} p_{s,g} \ge 1-p_r$ implies that every $p_{s,a}(X_i), p_{s,g}(X_{i,1}, X_{i,n+2}) \ge 1-p_r \geq 1-\bar{p}_r$. Hence, $1-p_{s,a}(X_i), 1-p_{s,g}(X_{i,1}, X_{i,n+2}) \leq \bar{p}_r$. Together with \eqref{eq:surrogate_to_true}, \eqref{eq:survival_uav}, \eqref{eq:survival_ugv}, this inequality implies:
\begin{subequations}
\label{eq:surrogate_to_true_2}
\begin{align}
P\!\left(\tau_a(X_i) \leq \bar{\tau}_a\right) \geq p_{s,a}(X_i),
\label{eq:surrogate_to_true_a_2}\\
P\!\left(\tau_g(X_{i,1},X_{i,n+2}) \leq \bar{\tau}_a\right) \geq p_{s,g}(X_{i,1}, X_{i,n+2}).
\label{eq:surrogate_to_true_g_2}
\end{align}
\end{subequations}

Using \eqref{eq:surrogate_to_true_2} and $\prod_i p_{s,a} {\cdot} p_{s,g} \ge 1-p_r$, we obtain:
\begin{equation}
    \label{eq:independent_risk}
    \prod_i P(\tau_a(X_i) \leq \bar{\tau}_a) \cdot P(\tau_g(X_{i,1},X_{i,n+2}) \leq \bar{\tau}_a) \geq 1 - p_r.
\end{equation}

Under Assumption~\ref{as:edges}, \eqref{eq:risk} is equivalent to \eqref{eq:independent_risk} due to the mutual independence of travel times. Hence the solution satisfies \eqref{eq:risk}.
Release and collect points $\tilde{X}_{i,1}$ and
$\tilde{X}_{i,n+2}$ lie on feasible ground via $\pi$, satisfying \eqref{eq:ground};
the remaining entries of $\tilde{X}_{1-m,:}$ are points of
$\tilde{\mathcal{P}}_{\text{UAV}}$ or repeated air points, and
$\tilde{X}_{m+1,:} = \tilde{x}_\text{f}$, satisfying \eqref{eq:air}. If no feasible
partition exists for $m \leq \tilde{n}$, the algorithm returns $\emptyset$.
Therefore, for $p_r \leq \bar{p}_r$, Alg.~\ref{alg:offline_planning} returns a
feasible solution to \eqref{eq:problem}, if one exists.

\textit{Complexity:} Step~1 solves a TSP on $\tilde{n}$ points, which runs in
$O(\tilde{n}^3)$ via Christofides' algorithm~\cite{christofides2022worst}. Step 2
builds $D_{\tilde{n},m}$ by sweeping $m$ from $1$ to $m_{min}$ (outer loop,
$m_{min} \leq \tilde{n}$). For each $m$, the inner loop iterates over
$j \in \{m{:}\tilde{n}\}$ with $\leq \tilde{n}$ iterations. The maximization over
$(k_{m-1}, r_m, c_m)$ requires $\leq \tilde{n}$ iterations for $k_{m-1}$ and
worst-case $\tilde{n}$ iterations each for $r_m, c_m$. The total is then
$O(\tilde{n}^4 m_{min})$. Since worst-case for $m_{min}$ is $m_{min} = \tilde{n}$,
overall complexity becomes $O(\tilde{n}^5)$. The projection function $\pi$ runs in
$O(\tilde{n}b)$ total for $\tilde{n}$ points and $b$ known convex obstacles (each
with a bounded number of faces), as each projection reduces to a convex quadratic
program per face~\cite{QP}, which is dominated by Step 2 assuming $b$ is bounded.
Therefore, the worst-case time complexity of Alg.~\ref{alg:offline_planning} is
$O(\tilde{n}^5)$.
\end{proof}

\section{Numerical Evaluations}
\label{sec:numerical}
We compare the scalability and performance of Alg.\ref{alg:offline_planning} to four different algorithms: 1) Branch and Cut \cite{mitchell2002branch}, 2) a metaheuristic, Simulated Annealing, and 3) two heuristics from literature \cite{ropero2019terra, er2025rspectrobustscalableplanner}. Python 3.9 running on an Intel Ultra 7 155h (22 processors) with 32GB RAM was used.

As environmental models require bounded distributions, the Uniform distribution ($\mathcal{U}$) is used. Travel time along an edge of length $l$ is modeled as $t{\sim}l{\cdot} \mathcal{U}\left(\mu{-}\sqrt{3}\sigma, \mu{+}\sqrt{3}\sigma\right)$ where the constant $\sqrt{3}$ scales for variance. The planning model follows: $M_\theta{=}l \cdot \mu$, $V_\theta{=}l^2 \cdot \sigma^2$, and independent edge times. The UAV model uses $\mu {=} 0.1 \frac{s}{m}$ and $\sigma {=} 0.01 \frac{\sqrt{s}}{m}$, and the UGV model uses $\mu {=} 0.4 \frac{s}{m}$ and $\sigma {=} 0.04 \frac{\sqrt{s}}{m}$.
Recharging is modeled as:
\begin{align}
    \tau_c(X_i) = \gamma \cdot \max(\tau_a(X_i), \tau_g(X_{i,1}, X_{i,n+2})),  \gamma \geq 0,
\end{align}
where $\gamma$ is a recharge ratio. 

For $n {\in} \{25, 50, 75, 100\}$ and $p_r{=}0.1$, 10 random $\mathcal{P}_{\text{UAV}}$ were created 
inside the environment defined in \eqref{eq:env} with $\bar{x} = \bar{y} = 4000m$, and fixed $z$-coordinate equal to $100m$. Since UAVs' horizontal speeds are greater than their vertical speeds \cite{guo2020precision}, the UAV vertical travel time is scaled by a factor of 5, i.e., $t_v {\sim} z {\cdot} 5{\cdot}\mathcal{U}\left(\mu{-}\sqrt{3}\sigma, \mu{+}\sqrt{3}\sigma\right)$. For each value of $n$, computation time and mission time ($\tau(X)$) were reported with mean $\pm$ standard deviation over the 10 realizations. Unless otherwise specified, all results are from offline planning only. Remaining parameters are in Table~\ref{tab:mission_parameters}.

\begin{table}[htpb]
\centering
\caption{Mission Parameters for Numerical Evaluations}
\label{tab:mission_parameters}
\resizebox{\columnwidth}{!}{
\begin{tabular}{|c|c||c|c|}
\hline
\textbf{Parameter} & \textbf{Value} & \textbf{Parameter} & \textbf{Value} \\
\hline
$x_{0}$ (Initial point) & $(0, 0, 0)$ & $x_{\text{f}}$ (Final point) & $(4000, 4000, 0)$ \\
\hline
$\bar{\tau}_a$ (Max flight time) & $600$ s & $\gamma$ (Recharge ratio) & $1$ \\
\hline
\end{tabular}}
\end{table}

\subsection{Scalability}
Table \ref{tab:scalability} demonstrates the scalability of PRO-SPECT for different $n$. Step 2 of PRO-SPECT is well-suited to parallelization, as $D_{:,m}$ can be solved independently across $m$. Results are reported for both serial and parallel implementations. Computation time increases with $n$ as expected from the worst-case $O(n^5)$ complexity (see Theorem \ref{th:feas_complex}). Empirical fits for serial and parallel implementations match $O(n^{4.25})$ and $O(n^{3.05})$, respectively. The results demonstrate PRO-SPECT's ability to handle a large range of problem sizes.

\begin{table}[htpb]
\centering
\caption{PRO-SPECT Comp. Time and Mission Time, $\tau(X)$, Varying $n$}
\resizebox{\columnwidth}{!}{%
\begin{tabular}{|c||c|c|c|}
\hline
$\boldsymbol{n}$ & \textbf{Comp. Time (s)} & \textbf{Parallel Comp. Time (s)} & $\boldsymbol{\tau(X)}$ (s) \\ \hline
25  & $8.1\pm0.21$  & $4.7\pm1.2$ & $5400\pm670$ \\ \hline
50  & $120\pm1.1$    & $25\pm6.5$   & $6500\pm570$ \\ \hline
75  & $850\pm1.6$    & $110\pm31$   & $7600\pm560$ \\ \hline
100 & $2800\pm7.5$   & $330\pm14$   & $8200\pm510$ \\ \hline
\end{tabular}
}
\label{tab:scalability}
\end{table}

\subsection{Validation of Probabilistic Constraints}
Table~\ref{tab:failure_rate} reports the mean empirical failure rate ($\hat{p}_r$) of PRO-SPECT for different $n$ and $p_r$. For each $(n, p_r)$ combination, a plan is generated for each of the 10 random $\mathcal{P}_\text{UAV}$ instances and executed 1000 times under independent stochastic travel time realizations to estimate mean $\hat{p}_r$. Then, $\hat{p}_r$ is reported across the 10 instances. In all cases, $\hat{p}_r$ remains below the user-specified $p_r$, validating \eqref{eq:risk}. 
\begin{table}[htpb]
\centering
\caption{\scriptsize PRO-SPECT Mean Empirical Failure Rate, $\hat{p}_r$, Varying $n$ and $p_r$}
\begin{tabular}{|c||c|c|c|c|}
\hline
\multirow{2}{*}{$\boldsymbol{n}$} & \multicolumn{4}{c|}{$\boldsymbol{p_r}$ (user-specified)} \\
\cline{2-5}
& \textbf{0.01} & \textbf{0.10} & \textbf{0.20} & \textbf{0.50} \\
\hline
\textbf{25}  & $0.0001$ & $0.018$ & $0.018$ & $0.10$ \\
\hline
\textbf{50}  & $0.0007$ & $0.011$ & $0.011$ & $0.033$ \\
\hline
\textbf{75}  & $0.0005$ & $0.015$ & $0.024$ & $0.11$ \\
\hline
\textbf{100} & $0.0001$ & $0.0061$ & $0.0061$ & $0.11$ \\
\hline
\end{tabular}
\label{tab:failure_rate}
\end{table}
\subsection{Online Re-planning}
Table~\ref{tab:online} evaluates the online re-planning component of PRO-SPECT against the offline-only plan in terms of empirical failure rate $\hat{p}_r$ and expected mission time $\tau(X)$. For each $n{\in}{\{20, 30, 40, 50\}}$, a plan is generated for each of the 10 random $\mathcal{P}_\text{UAV}$ instances and executed 25 times under independent stochastic travel time realizations. The mean $\hat{p}_r$ is then reported across the 10 instances. Re-planning is triggered after each air point, with earlier interventions yielding larger reductions in $\hat{p}_r$. The re-planning horizon $\tilde{m}$ controls the trade-off between solution quality and computation time; all runs use $\tilde{m} = 2$. Runs with online planning rarely encounter failure due to the ability to end a tour before serious risk. This has a mixed effect on mission time, as the planner can also save time when available. Overall the online planner provides flexibility and maintains $\hat{p}_r{<}p_r$.
\begin{table}[htpb]
\centering
\caption{PRO-SPECT Offline vs. Online: Empirical Failure Rate, $\hat{p}_r$, and Empirical Successful Mission Time, $\hat{\tau}(X)$, for $p_r=0.1$ and Varied $n$}
\resizebox{0.8\columnwidth}{!}{%
\begin{tabular}{|c||c|c|c|c|}
\hline
\multirow{2}{*}{$\boldsymbol{n}$} & \multicolumn{2}{c|}{\textbf{Offline}} & \multicolumn{2}{c|}{\textbf{Online}} \\
\cline{2-5}
& $\boldsymbol{\hat{p}_r}$ & $\boldsymbol{\hat{\tau}(X)}$ \textbf{(s)} & $\boldsymbol{\hat{p}_r}$ & $\boldsymbol{\hat{\tau}(X)}$ \textbf{(s)} \\
\hline
\textbf{20}  & $0.012$ & $5300 \pm 650$ & $\boldsymbol{0}$ & $\boldsymbol{5100 \pm 510}$ \\
\hline
\textbf{30}  & $0.020$ & $5800 \pm 760$ & $\boldsymbol{0}$ & $5800 \pm 1000$ \\
\hline
\textbf{40}  & $0.0058$ & $\boldsymbol{5800 \pm 420}$ & $\boldsymbol{0.004}$ & $6200 \pm 560$ \\
\hline
\textbf{50} & $0.011$ & $6500 \pm 570$ & $\boldsymbol{0}$ & $6500 \pm 450$ \\
\hline
\end{tabular}}
\label{tab:online}
\end{table}
\subsection{Comparison to Branch and Cut}
Table~\ref{tab:bnc} compares PRO-SPECT against Branch and Cut, which returns the global optimum for \eqref{eq:problem}. Due to the NP-hardness of \eqref{eq:problem}, Branch and Cut computation time grows exponentially with $n$, limiting its applicability to very small $n$. In contrast, PRO-SPECT achieves competitive mission times at a fraction of the computation cost.
The optimality gaps in expected mission time are 20.8\%, 10.7\%, and 5.7\% for $n=2,3,4$, respectively.
\begin{table}[htpb]
\centering
\caption{Branch and Cut vs. PRO-SPECT: Computation Time \\ and Mission Time, $\tau(X)$, Varying $n$}
\resizebox{\columnwidth}{!}{%
\begin{tabular}{|c||c|c|c|c|}
\hline
\multirow{2}{*}{$\boldsymbol{n}$} & \multicolumn{2}{c|}{\textbf{Branch and Cut}} & \multicolumn{2}{c|}{\textbf{PRO-SPECT}} \\ \cline{2-5} 
 & \textbf{Comp. Time (s)} & $\boldsymbol{\tau(X) \text{ (s)}}$ & \textbf{Comp. Time (s)} & $\boldsymbol{\tau(X) \text{ (s)}}$ \\ \hline
\textbf{2} & $0.43 \pm 0.36$ & $\boldsymbol{2400 \pm 140}$ & $\boldsymbol{0.0051 \pm 0.0007}$ & $2900 \pm 640$ \\ \hline
\textbf{3} & $3.0 \pm 1.3$ & $\boldsymbol{2800 \pm 510}$ & $\boldsymbol{0.0075 \pm 0.0016}$ & $3100 \pm 660$ \\ \hline
\textbf{4} & $18 \pm 14$ & $\boldsymbol{3500 \pm 900}$ & $\boldsymbol{0.015 \pm 0.0049}$ & $3700 \pm 980$ \\ \hline
\end{tabular}}
\label{tab:bnc}
\end{table}

\subsection{Comparison to Simulated Annealing}
Table~\ref{tab:sa} compares PRO-SPECT against Simulated Annealing (SA) for $n=10$ across 10 random $\mathcal{P}_\text{UAV}$ instances. A logarithmic cooling schedule was used for SA \cite{nourani1998sa} with $T_\text{min} {=} 0.01$ and $T_\text{max} \in \{500, 1000, 1500\}$. Each instance was run 3 times due to the stochastic nature of SA. While increasing the number of steps increases computation cost, solution quality does not improve consistently, suggesting diminishing returns. PRO-SPECT achieves significantly better computation and mission times: $0.13 \pm 0.05$ s and $\tau(X) = 4400 \pm 1300$ s respectively.

\begin{table}[htpb]
\centering
\caption{Simulated Annealing Average Computation Time and Mission Time, $\tau(X)$, Varying $T_{\text{max}}$ and Number of Steps}
\resizebox{\linewidth}{!}{
\begin{tabular}{|c||c|c|c|c|c|c|}
\hline
\multirow{3}{*}{$\boldsymbol{T_{\text{max}}}$}
& \multicolumn{3}{c|}{\textbf{Comp. Time (s)}}
& \multicolumn{3}{c|}{$\boldsymbol{\tau(X)}$ \textbf{(s)}} \\
\cline{2-7}
& \multicolumn{3}{c|}{\textbf{Steps}}
& \multicolumn{3}{c|}{\textbf{Steps}} \\
\cline{2-7}
& \textbf{50k} & \textbf{100k} & \textbf{150k}
& \textbf{50k} & \textbf{100k} & \textbf{150k} \\
\hline
\textbf{500}  & $40 \pm 1.5$ & $40 \pm 3.7$ & $41 \pm 3.5$ & $8700 \pm 1200$ & $9100 \pm 1300$ & $8700 \pm 1100$ \\
\hline
\textbf{1000} & $80 \pm 14$ & $80 \pm 11$ & $79 \pm 8.2$ & $7900 \pm 830$ & $7800 \pm 1300$ & $8300 \pm 1100$ \\
\hline
\textbf{1500} & $120 \pm 7.0$ & $110 \pm 9.8$ & $120 \pm 8.4$ & $7700 \pm 1000$ & $7700 \pm 1100$ & $8000 \pm 820$ \\
\hline
\end{tabular}%
}
\label{tab:sa}
\end{table}

\subsection{Comparison to State-of-the-Art}
PRO-SPECT is compared with two methods that address similar energy-aware UAV-UGV planning problems: 1) TERRA \cite{ropero2019terra}, and 2) RSPECT \cite{er2025rspectrobustscalableplanner}. Since the methods are deterministic, their planning budgets are scaled conservatively i.e., reducing $\bar{\tau}_a$ during planning, so that the deterministic plan accounts for likely travel time deviations. Plans are executed under the full stochastic travel times. All methods use parameters from Table~\ref{tab:mission_parameters}, unless otherwise specified. TSPs and GTSPs are solved using \cite{christofides2022worst} and \cite{smith2016GLNS}, respectively.

\subsubsection{Comparison to TERRA}
TERRA \cite{ropero2019terra} is deterministic and assumes a known environment; its planning UAV range is scaled by $\frac{\mu}{\mu+\sqrt{3}\sigma}$ (i.e., conservative planning bounds) to provide a more fair comparison under stochastic travel times. Note that if runtime conditions still violate the actual flight time ($\bar{\tau}_a$) TERRA might require complete mission re-planning, whereas PRO-SPECT can perform online re-planning. As TERRA does not support distinct $x_\text{f}$ and $x_0$, we set $x_\text{f} = x_0 = (0,0,0)$. A plan is generated for each of the 10 $\mathcal{P}_\text{UAV}$ and executed 100 times via independent stochastic travel time realizations to obtain mean $\hat{p}_r$.

TERRA uses rendezvous points (UGV points), where the UGV remains stationary while the UAV visits monitoring points ($\mathcal{P}_{\text{UAV}}$) and meets with the UGV. This design choice can significantly slow missions when monitoring points are far apart. PRO-SPECT addresses this by allowing the UGV to move concurrently with the UAV. As shown in Table~\ref{tab:sota_terra}, PRO-SPECT consistently achieves lower mission times and substantially lower failure rates that always satisfy $\hat{p}_r < p_r$, while TERRA's $\hat{p}_r$ reaches up to 0.85 in some configurations.

\begin{table}[htpb]
\centering
\caption{TERRA vs. PRO-SPECT: Computation Time, Mission Time, $\tau(X)$, and Empirical Failure Rate, $\hat{p}_r$, Varying $n$}
\resizebox{\columnwidth}{!}{%
\begin{tabular}{|c||c|c|c|c|c|c|}
\hline
\multirow{2}{*}{$\boldsymbol{n}$} & \multicolumn{3}{c|}{\textbf{TERRA}} & \multicolumn{3}{c|}{\textbf{PRO-SPECT}} \\
\cline{2-7}
& \textbf{Comp. Time (s)} & $\boldsymbol{\tau(X)}$ \textbf{(s)} & $\boldsymbol{\hat{p}_r}$ & \textbf{Comp. Time (s)} & $\boldsymbol{\tau(X)}$ \textbf{(s)} & $\boldsymbol{\hat{p}_r}$ \\
\hline
\textbf{25}  & $\boldsymbol{0.017 \pm 0.0069}$ & $12000 \pm 1100$ & $0.076$ & $4.0 \pm 0.21$ & $\boldsymbol{6200 \pm 500}$ & $\textbf{0.0095}$ \\
\hline
\textbf{50}  & $\boldsymbol{0.09 \pm 0.044}$ & $16000 \pm 1400$ & $0.31$ & $68 \pm 2.7$ & $\boldsymbol{6900 \pm 780}$ & $\textbf{0.014}$ \\
\hline
\textbf{75}  & $\boldsymbol{0.31 \pm 0.23}$ & $19000 \pm 1600$ & $0.68$ & $380 \pm 6.5$ & $\boldsymbol{7600 \pm 790}$ & $\textbf{0.014}$ \\
\hline
\textbf{100} & $\boldsymbol{0.82 \pm 0.37}$ & $22000 \pm 1900$ & $0.85$ & $1300 \pm 14$ & $\boldsymbol{8600 \pm 480}$ & $\textbf{0.0013}$ \\
\hline
\end{tabular}}
\label{tab:sota_terra}
\end{table}

\subsubsection{Comparison to RSPECT}
Unlike TERRA, RSPECT \cite{er2025rspectrobustscalableplanner} supports distinct $x_\text{f}$ and $x_0$, and does not require a known environment a priori. RSPECT handles uncertainty through its deterministic margins $(\delta_a,\delta_g)$, offering lower bounds on plan robustness; however, it cannot bound the probability of failure to a user-specified risk level. In contrast, PRO-SPECT uses a joint chance constraint, directly bounding the probability of failure globally to  $p_r$.

RSPECT's $\delta_a$ and $\delta_g$ are set by estimating $p_r$ as a fraction of the uniform distribution, which is then used to scale $\bar{\tau}_a$ (i.e., conservative planning bounds).
This assumes the worst-case travel time $\mu{+}\sqrt{3}\sigma$ for $p_r{=}0$; the mean $\mu$ at $p_r{=}0.5$; and the best-case $\mu{-}\sqrt{3}\sigma$ at $p_r{=}1$:
\begin{equation}
    \delta_* = \bar{\tau}_a \left(1 - \frac{\mu}{\mu + \sqrt{3} \sigma (1 - 2 p_r)} \right).
    \label{eq:approximation}
\end{equation}
This approximation is necessitated by RSPECT's lack of a probabilistic planning mechanism: without a chance constraint, this represents the closest feasible equivalence under the assumed uniform distribution.

For each combination of $n$ and $p_r\in\{0.01, 0.1, 0.2, 0.5\}$ a plan is generated for each of the 10 $\mathcal{P}_\text{UAV}$ instances, and executed 1000 times under independent stochastic travel time realizations to estimate mean $\hat{p}_r$.

As shown in Table \ref{tab:sota_rspect}, computation time is independent of $p_r$ for both RSPECT and PRO-SPECT. For RSPECT, this is expected as \eqref{eq:approximation} is used. For PRO-SPECT, this is a stronger result. Since DP optimizes under a joint chance constraint, tighter $p_r$ would be expected to demand more tours and computation effort. However, this effect is weak, as the increase in number of tours remains modest.

As demonstrated in all cases of Table \ref{tab:sota_rspect}, PRO-SPECT achieves mission times up to $10\%$ lower than RSPECT, while maintaining $\hat{p}_r {<} p_r$. The difference is more evident in higher $p_r$ cases: at $p_r{=}0.5$ and $n{=}100$, RSPECT's $\hat{p}_r$ rises to 0.9; while PRO-SPECT maintains $\hat{p}_r{=}0.11$, within the specified $p_r$. This is an important limitation: the deterministic margin cannot adapt to specified $p_r$, causing $\hat{p}_r$ to diverge. RSPECT achieves much lower computation times and has an $O(n^3)$ time complexity, which permits a design decision between plan efficiency and safety, and computation cost.

\begin{table}[htpb]
\centering
\caption{RSPECT vs. PRO-SPECT: Comp. Time, Mission Time, $\tau(X)$, and Empirical Failure Rate, $\hat{p}_r$, Varying $n$, $p_r$}
\resizebox{\columnwidth}{!}{%
\begin{tabular}{|c|c||c|c|c|c|c|c|}

\hline
\multirow{2}{*}{$\boldsymbol{p_r}$} & \multirow{2}{*}{$\boldsymbol{n}$} & \multicolumn{3}{c|}{\textbf{RSPECT}} & \multicolumn{3}{c|}{\textbf{PRO-SPECT}} \\
\cline{3-8}
& & \textbf{Comp. Time (s)} & $\boldsymbol{\tau(X)}$ \textbf{(s)} & $\boldsymbol{\hat{p}_r}$ & \textbf{Comp. Time (s)} & $\boldsymbol{\tau(X)}$ \textbf{(s)} & $\boldsymbol{\hat{p}_r}$ \\
\hline
\multirow{4}{*}{\textbf{0.01}}
& \textbf{25}  & $\boldsymbol{ 0.057 \pm 0.012 }$ & $5900 \pm 750$ & $ \textbf{0} $ & $6.6 \pm 0.11$ & $\boldsymbol{5800 \pm 720}$ & $0.0001$ \\
\cline{2-8}
& \textbf{50}  & $\boldsymbol{ 0.16 \pm 0.017 }$ & $7100 \pm 640$ & $ \textbf{0} $ & $140 \pm 8.8$ & $\boldsymbol{6700 \pm 530}$ & $0.0021$ \\
\cline{2-8}
& \textbf{75}  & $\boldsymbol{ 0.32 \pm 0.048 }$ & $8100 \pm 490$ & $ \textbf{0.0006} $ & $790 \pm 2.5$ & $\boldsymbol{7700 \pm 440}$ & $0.0010$ \\
\cline{2-8}
& \textbf{100} & $\boldsymbol{ 0.52 \pm 0.098 }$ & $9000 \pm 600$ & $ 0.0013 $ & $3300 \pm 33$ & $\boldsymbol{8400 \pm 640}$ & $\textbf{0.0003}$ \\
\hline \hline

\multirow{4}{*}{\textbf{0.1}}
& \textbf{25}  & $\boldsymbol{ 0.055 \pm 0.012 }$ & $5800 \pm 590$ & $ 0.025 $ & $8.1 \pm 0.21$ & $\boldsymbol{5400 \pm 670}$ & $\textbf{0.018}$ \\
\cline{2-8}
& \textbf{50}  & $\boldsymbol{ 0.14 \pm 0.02 }$ & $7100 \pm 820$ & $ 0.029 $ & $120 \pm 1.1$ & $\boldsymbol{6500 \pm 570}$ & $\textbf{0.011}$ \\
\cline{2-8}
& \textbf{75}  & $\boldsymbol{ 0.32 \pm 0.05 }$ & $8200 \pm 380$ & $ 0.025 $ & $850 \pm 1.6$ & $\boldsymbol{7600 \pm 560}$ & $\textbf{0.015}$ \\
\cline{2-8}
& \textbf{100} & $\boldsymbol{ 0.53 \pm 0.075 }$ & $8800 \pm 520$ & $ 0.027 $ & $2800 \pm 7.5$ & $\boldsymbol{8200 \pm 510}$ & $\textbf{0.0061}$ \\
\hline \hline

\multirow{4}{*}{\textbf{0.2}}
& \textbf{25}  & $\boldsymbol{ 0.052 \pm 0.005 }$ & $6100 \pm 530$ & $ 0.040 $ & $7.7 \pm 0.13$ & $\boldsymbol{5400 \pm 670}$ & $\textbf{0.018}$ \\
\cline{2-8}
& \textbf{50}  & $\boldsymbol{ 0.15 \pm 0.030 }$ & $6900 \pm 550$ & $ 0.075 $ & $120 \pm 1.5$ & $\boldsymbol{6500 \pm 570}$ & $\textbf{0.011}$ \\
\cline{2-8}
& \textbf{75}  & $\boldsymbol{ 0.32 \pm 0.053 }$ & $7800 \pm 600$ & $ 0.15 $ & $800 \pm 2$ & $\boldsymbol{7500 \pm 620}$ & $\textbf{0.024}$ \\
\cline{2-8}
& \textbf{100} & $\boldsymbol{ 0.54 \pm 0.096 }$ & $8500 \pm 560$ & $ 0.13 $ & $2900 \pm 10$ & $\boldsymbol{8200 \pm 510}$ & $\textbf{0.0061}$ \\
\hline \hline

\multirow{4}{*}{\textbf{0.5}}
& \textbf{25}  & $\boldsymbol{ 0.051 \pm 0.0055 }$ & $5500 \pm 600$ & $ 0.49 $ & $6.2 \pm 0.28$ & $\boldsymbol{5200 \pm 550}$ & $\textbf{0.10}$ \\
\cline{2-8}
& \textbf{50}  & $\boldsymbol{ 0.13 \pm 0.017 }$ & $6600 \pm 430$ & $ 0.63 $ & $110 \pm 1.3$ & $\boldsymbol{6500 \pm 620}$ & $\textbf{0.033}$ \\
\cline{2-8}
& \textbf{75}  & $\boldsymbol{ 0.30 \pm 0.041 }$ & $7600 \pm 640$ & $ 0.74 $ & $800 \pm 2.6$ & $\boldsymbol{7400 \pm 610}$ & $\textbf{0.11}$ \\
\cline{2-8}
& \textbf{100} & $\boldsymbol{ 0.53 \pm 0.085 }$ & $8300 \pm 550$ & $ 0.90 $ & $3100 \pm 20$ & $\boldsymbol{8100 \pm 470}$ & $\textbf{0.11}$ \\
\hline

\end{tabular}}
\label{tab:sota_rspect}
\end{table}

\section{Simulations}

\subsection{Setup}
To validate PRO-SPECT, we simulate the UAV-UGV team in ROS2 using a ROSflight-based UAV (ROScopter)~\cite{moore2025rosflight20leanros} and a Clearpath Jackal UGV. Stochasticity is introduced via an unknown zero-mean spatial wind field generated by Perlin Noise \cite{perlin1985image}. This distribution produces smooth vector fields with spatial self-correlation, consistent with structured atmospheric disturbances, providing a suitable stochastic environment. $\mathcal{P}_{\text{UAV}}$ instances were created inside \eqref{eq:env} with $\bar{x}{=}\bar{y}{=}40$m and fixed $z{=}5$m, with $n{=}10$ and $p_r{=}0.1$ repeated over $10$ trials. Online re-planning is enabled with horizon $\tilde{m}{=}2$. Simulation parameters are presented in Table \ref{tab:mission_parameters_ros}.
\begin{table}[htpb]
\centering
\caption{Mission Parameters for Simulations}
\label{tab:mission_parameters_ros}
\resizebox{\columnwidth}{!}{
\begin{tabular}{|c|c||c|c|}
\hline
\textbf{Parameter} & \textbf{Value} & \textbf{Parameter} & \textbf{Value} \\
\hline
$x_{0}$ (Initial point) & $(0, 0, 0)$ & $x_{\text{f}}$ (Final point) & $(40, 40, 0)$ \\
\hline
$\bar{\tau}_a$ (Max flight time) & $60$ s & $\gamma$ (Recharge ratio) & $1$ \\
\hline
$\bar{w}_x, \bar{w}_y$ (Max wind speed, x, y) & $1$ m/s & $\bar{w}_z$ (Max wind speed, z) & $0.3$ m/s \\
\hline
$[\underline{v}_a, \overline{v}_a]$ (UAV speed) & $[1.5,2]$ m/s &$[\underline{v}_g, \overline{v}_g]$ (UGV speed)  & $[0.15,0.25]$ m/s \\
\hline
\end{tabular}}
\end{table}

\subsubsection{Agent Models}

To connect the travel models~\eqref{eq:uav_model}, \eqref{eq:ugv_model} to the spatial wind field, we use the fact that wind speed is bounded. Consider a UAV edge from $u$ to $v$ with length $L$ in direction $\hat{d}$ with closed-loop along-track speed $v_a \in [\underline{v}_a, \overline{v}_a]$ under PID 
tracking. Given component-wise wind bounds $|w_x|\le\bar{w}_x$, $|w_y|\le\bar{w}_y$, $|w_z|\le\bar{w}_z$ (Table~\ref{tab:mission_parameters_ros}), the along-track component satisfies $|\hat{d}^\top w|\le\bar{w}_\parallel$, where $\bar{w}_\parallel {:=} |\hat{d}_x|\bar{w}_x + |\hat{d}_y|\bar{w}_y + |\hat{d}_z|\bar{w}_z$. Assuming $\underline{v}_a > \bar{w}_\parallel$,\footnote{Unbounded wind may make the problem infeasible.} 
the UAV edge travel time is bounded:
\begin{equation}
    t_{\min,a}(u,v)=\frac{L}{\overline{v}_a+\bar{w}_\parallel}, \quad
    t_{\max,a}(u,v)=\frac{L}{\underline{v}_a-\bar{w}_\parallel}.
    \label{eq:uav_time_bounds}
\end{equation}

Since only bounds are available, we model each edge time as uniform on $[t_{\min,a}(u,v),\, t_{\max,a}(u,v)]$, giving the moments in~\eqref{eq:uav_model}. Because the wind is bounded, the edge travel times are bounded, satisfying the boundedness of Assumption~\ref{as:edges}. Absent a known wind distribution, we assume edge-time independence. The Perlin field is spatially correlated, so these simulations test the method beyond the independence assumption of Theorem~\ref{th:feas_complex}; positive correlation is expected to inflate the aggregate variance, which would tend to keep the independent-edge estimate conservative.
\begin{equation}
    M_{\theta,a}(u,v)=\frac{t_{\min.a}(u,v)+t_{\max,a}(u,v)}{2}, \quad
    V_{\theta,a}(u,v)=\frac{\big(t_{\max,a}(u,v)-t_{\min,a}(u,v)\big)^2}{12}.
    \label{eq:uav_moments_bounds}
\end{equation}
The UGV moments are obtained similarly, with $t_{\min,g}(u,v){=}L/\overline{v}_g$ and $t_{\max,g}(u,v){=}L/\underline{v}_g$ from $v_g{\in}[\underline{v}_g, \overline{v}_g]$ under PID tracking.

\begin{figure}[htbp]
    \centering
    \includegraphics[width=1.0\linewidth]{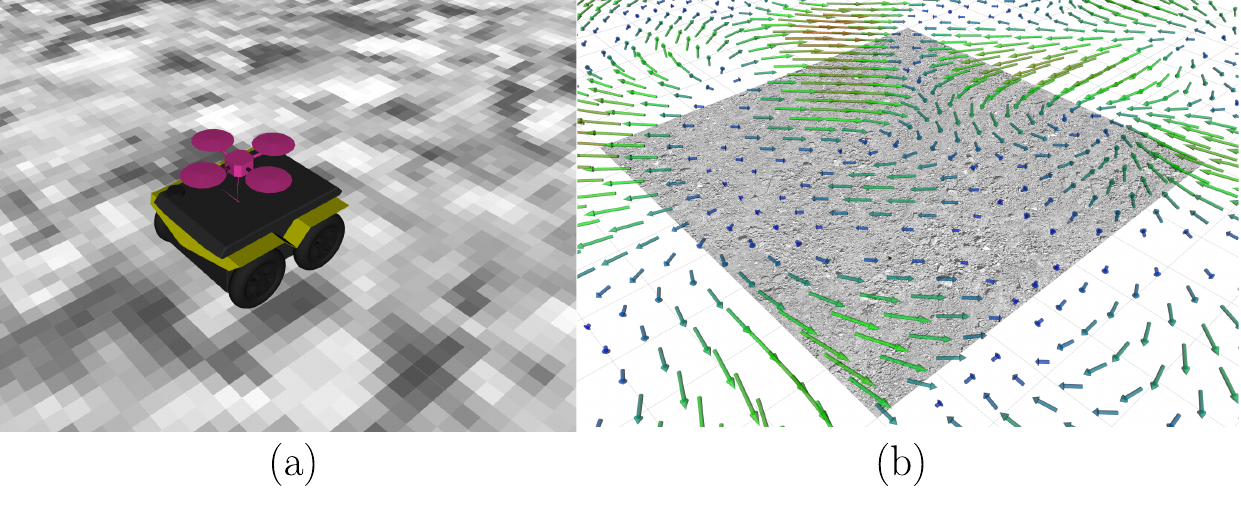}
    \caption{Simulation setup. (a) UAV-UGV team in the simulated environment. (b) Simulated wind field applied to the UAV, with arrow direction and magnitude indicating local wind velocity.}
    \label{fig:placeholder}
\end{figure}

\subsection{Results}
Across all 10 trials, one trial experienced one tour time above $\bar{\tau}_a$, giving total $\hat{p}_r{=}0.1$ and maintaining the specified $p_r{=}0.1$. The mean of the ratio of empirical to planned mission time $\frac{\hat{\tau}(X)}{\tau(X)}$ was $1.02\pm0.09$ across successful trials.

Fig.~\ref{fig:plan_ros} illustrates a representative run: at $t{=}0$ the full plan ($m{=}4$) is initialized (Fig.~\ref{fig:plan0}), and as the mission progresses the online re-planner updates after each UAV point to refine the remaining tours in response to wind-induced disturbances. Re-planning is visible in the evolving plans at $t{=}150$~s and $t{=}410$~s: (Fig.~\ref{fig:plan1}) and (Fig.~\ref{fig:plan2}). 

\begin{figure}[htb]
\centering
\begin{subfigure}[t]{0.3\linewidth}
    \includegraphics[width=\textwidth]{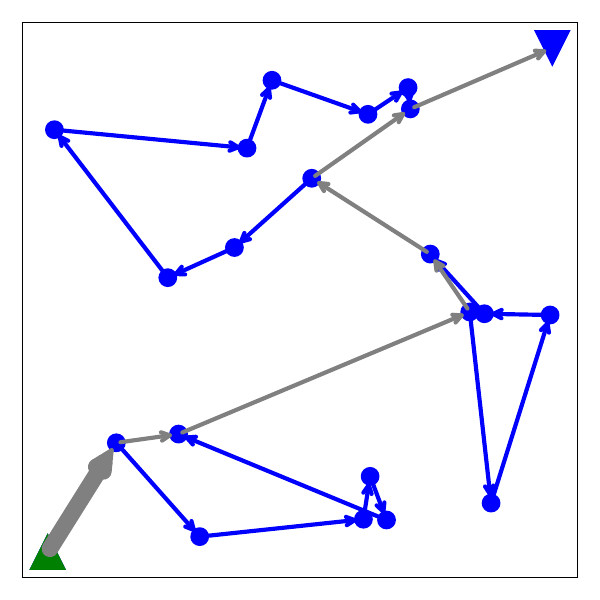}
    \caption{$t=0$}
    \label{fig:plan0}
\end{subfigure}
\hfill
\begin{subfigure}[t]{0.3\linewidth}
    \includegraphics[width=\textwidth]{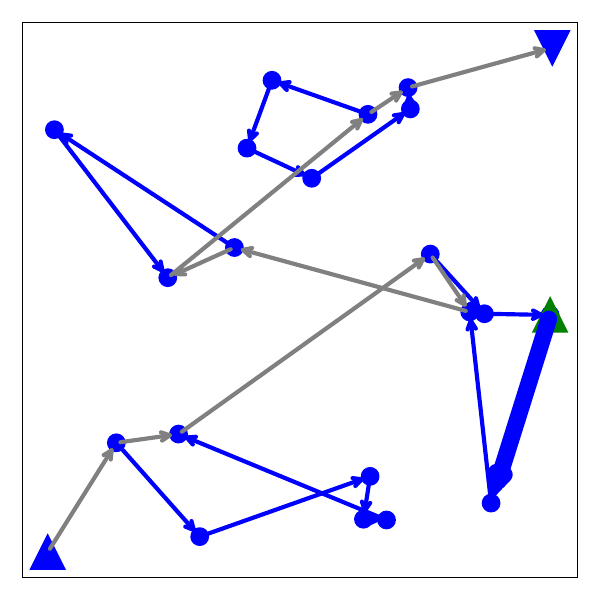}
    \caption{$t=150$}
    \label{fig:plan1}
\end{subfigure}
\hfill
\begin{subfigure}[t]{0.3\linewidth}
    \includegraphics[width=\textwidth]{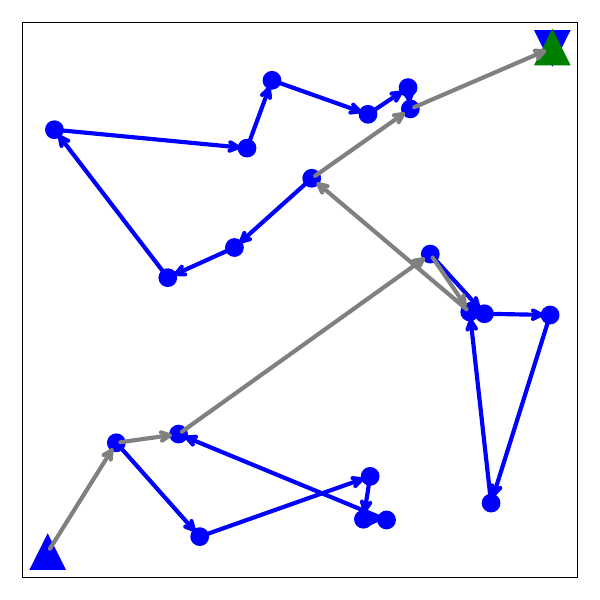}
    \caption{$t=410$}
    \label{fig:plan2}
\end{subfigure}
\caption[Mission Plan in ROS Simulation]{Mission Plan in ROS Simulation with Online Replanning: \eqref{fig:plan0} at mission start, \eqref{fig:plan1} partway through second tour, and \eqref{fig:plan2} at mission end. Green triangle denotes UAV location and bold edge denotes next transit.}
\label{fig:plan_ros}
\end{figure}

Less than $1\%$ of mission time was spent waiting for computation; online re-planning was fast enough to execute on the run. Therefore the increase of the mission time is largely due to wind-induced delay rather than re-planning overhead. \textcolor{black}{These results suggest that PRO-SPECT can still produce feasible solutions even under a spatially correlated wind field, a setting outside the assumption of Theorem~\ref{th:feas_complex}, with empirical risk meeting the target ($\hat{p}_r \leq p_r = 0.1$), and that online re-planning restores feasibility on the fly.}

\section{Conclusion}
We presented PRO-SPECT, a scalable algorithm for energy-aware UAV-UGV planning in stochastic environments. The problem was formulated as a Mixed-Integer Program with a joint chance constraint. The integrated offline-online framework combines TSP-based visit ordering with a dynamic program for tour construction, with an online re-planner that provably preserves the global risk bound under disturbances. Feasibility and worst-case time complexity were formally guaranteed (Thm.\ref{th:feas_complex}). Validations via numerical evaluations and ROS simulations were carried out. Failure rate was observed below the allowable limit in all cases, with mission times competitive with or better than all baselines.

Future work includes extending our method to multi-UAV fleets and incorporation of more complex (spatial, temporal, logical) planning specifications.

\bibliographystyle{ieeetr}
\bibliography{references}

@article{manfreda2018use,
  title={On the use of unmanned aerial systems for environmental monitoring},
  author={Manfreda, Salvatore and McCabe, Matthew F and Miller, Pauline E and Lucas, Richard and Pajuelo Madrigal, Victor and Mallinis, Giorgos and Ben Dor, Eyal and Helman, David and Estes, Lyndon and Ciraolo, Giuseppe and others},
  journal={Remote sensing},
  volume={10},
  number={4},
  pages={641},
  year={2018},
  publisher={MDPI}
}

@article{boccardo2015uav,
  title={UAV deployment exercise for mapping purposes: Evaluation of emergency response applications},
  author={Boccardo, Piero and Chiabrando, Filiberto and Dutto, Furio and Giulio Tonolo, Fabio and Lingua, Andrea},
  journal={Sensors},
  volume={15},
  number={7},
  pages={15717--15737},
  year={2015},
  publisher={mdpi}
}

@article{tokekar2016sensor,
  title={Sensor planning for a symbiotic UAV and UGV system for precision agriculture},
  author={Tokekar, Pratap and Vander Hook, Joshua and Mulla, David and Isler, Volkan},
  journal={IEEE transactions on robotics},
  volume={32},
  number={6},
  pages={1498--1511},
  year={2016},
  publisher={IEEE}
}

@inproceedings{shi2022risk,
  title={Risk-aware uav-ugv rendezvous with chance-constrained markov decision process},
  author={Shi, Guangyao and Karapetyan, Nare and Asghar, Ahmad Bilal and Reddinger, Jean-Paul and Dotterweich, James and Humann, James and Tokekar, Pratap},
  booktitle={2022 IEEE 61st Conference on Decision and Control (CDC)},
  pages={180--187},
  year={2022},
  organization={IEEE}
}

@INPROCEEDINGS{yu2021rl,
  author={Yu, Qifei and Shen, Zhexin and Pang, Yijiang and Liu, Rui},
  booktitle={2021 IEEE 17th International Conference on Automation Science and Engineering (CASE)}, 
  title={Proficiency Constrained Multi-Agent Reinforcement Learning for Environment-Adaptive Multi UAV-UGV Teaming}, 
  year={2021},
  volume={},
  number={},
  pages={2114-2118},
  doi={10.1109/CASE49439.2021.9551457}}

@inproceedings{christofides2022worst,
  title={Worst-case analysis of a new heuristic for the travelling salesman problem},
  author={Christofides, Nicos},
  booktitle={Operations Research Forum},
  volume={3},
  number={1},
  pages={20},
  year={2022},
  organization={Springer}
}

@ARTICLE{Lin2022,
  author={Lin, Xiaoshan and Yazıcıoğlu, Yasin and Aksaray, Derya},
  journal={IEEE Robotics and Automation Letters}, 
  title={Robust Planning for Persistent Surveillance With Energy-Constrained UAVs and Mobile Charging Stations}, 
  year={2022},
  volume={7},
  number={2},
  pages={4157-4164},
  doi={10.1109/LRA.2022.3146938}}

@article{ropero2019terra,
title = {TERRA: A path planning algorithm for cooperative UGV–UAV exploration},
journal = {Engineering Applications of Artificial Intelligence},
volume = {78},
pages = {260-272},
year = {2019},
issn = {0952-1976},
doi = {https://doi.org/10.1016/j.engappai.2018.11.008},
url = {https://www.sciencedirect.com/science/article/pii/S095219761830246X},
author = {Fernando Ropero and Pablo Muñoz and María D. R-Moreno},

}

@INPROCEEDINGS{yu2018algorithms,
  author={Yu, Kevin and Budhiraja, Ashish Kumar and Tokekar, Pratap},
  booktitle={2018 IEEE International Conference on Robotics and Automation (ICRA)}, 
  title={Algorithms for Routing of Unmanned Aerial Vehicles with Mobile Recharging Stations}, 
  year={2018},
  volume={},
  number={},
  pages={5720-5725},
  doi={10.1109/ICRA.2018.8460819}
}

@article{nourani1998sa,
doi = {10.1088/0305-4470/31/41/011},
url = {https://dx.doi.org/10.1088/0305-4470/31/41/011},
year = {1998},
month = {oct},
publisher = {},
volume = {31},
number = {41},
pages = {8373},
author = {Yaghout Nourani and Bjarne Andresen},
title = {A comparison of simulated annealing cooling strategies},
journal = {Journal of Physics A: Mathematical and General},
}

@article{mitchell2002branch,
  title={Branch-and-cut algorithms for combinatorial optimization problems},
  author={Mitchell, John E},
  journal={Handbook of applied optimization},
  volume={1},
  number={1},
  pages={65--77},
  year={2002},
  publisher={Oxford, UK},
  label={Mitchell2002}
}

@article{maini2019coverage,
  author={Maini, Parikshit and Sundar, Kaarthik and Singh, Mandeep and Rathinam, Sivakumar and Sujit, P. B.},
  journal={IEEE Transactions on Aerospace and Electronic Systems}, 
  title={Cooperative Aerial–Ground Vehicle Route Planning With Fuel Constraints for Coverage Applications}, 
  year={2019},
  volume={55},
  number={6},
  pages={3016-3028},
  doi={10.1109/TAES.2019.2917578}}

@article{QP,
  title={An extension of Karmarkar's projective algorithm for convex quadratic programming},
  author={Ye, Yinyu and Tse, E.},
  journal={Mathematical Programming},
  volume={44},
  number={1-3},
  pages={157--179},
  year={1989},
  publisher={Springer}
}

@Article{smith2016GLNS,
  author =    {S. L. Smith and F. Imeson},
  title =     {{GLNS}: An Effective Large Neighborhood Search Heuristic
  				for the Generalized Traveling Salesman Problem},
  journal =   {Computers \& Operations Research},
  volume =    87,
  pages =     {1-19},
  year =      2017,
}

@article{guo2020precision,
  title={Precision landing test and simulation of the agricultural UAV on apron},
  author={Guo, Yangyang and Guo, Jiaqian and Liu, Chang and Xiong, Hongting and Chai, Lilong and He, Dongjian},
  journal={Sensors},
  volume={20},
  number={12},
  pages={3369},
  year={2020},
  publisher={MDPI}
}

@misc{er2025rspectrobustscalableplanner,
      title={RSPECT: Robust and Scalable Planner for Energy-Aware Coordination of UAV-UGV Teams in Aerial Monitoring}, 
      author={Cahit Ikbal Er and Amin Kashiri and Yasin Yazicioglu},
      year={2025},
      eprint={2511.21957},
      archivePrefix={arXiv},
      primaryClass={cs.RO},
      url={https://arxiv.org/abs/2511.21957}, 
}

@article{thelasingha2024iterative,
  title={Iterative planning for multi-agent systems: An application in energy-aware UAV-UGV cooperative task site assignments},
  author={Thelasingha, Neelanga and Julius, A Agung and Humann, James and Reddinger, Jean-Paul and Dotterweich, James and Childers, Marshal},
  journal={IEEE Transactions on Automation Science and Engineering},
  volume={22},
  pages={3685--3703},
  year={2024},
  publisher={IEEE}
}

@misc{cai2025energyawareroutingalgorithmmobile,
      title={Energy-Aware Routing Algorithm for Mobile Ground-to-Air Charging}, 
      author={Bill Cai and Fei Lu and Lifeng Zhou},
      year={2025},
      eprint={2310.07729},
      archivePrefix={arXiv},
      primaryClass={cs.RO},
      url={https://arxiv.org/abs/2310.07729}, 
}

@article{albarakati2021multiobjective,
  title={Multiobjective risk-aware path planning in uncertain transient currents: An ensemble-based stochastic optimization approach},
  author={Albarakati, Sultan and Lima, Ricardo M and Theu{\ss}l, Thomas and Hoteit, Ibrahim and Knio, Omar},
  journal={IEEE Journal of Oceanic Engineering},
  volume={46},
  number={4},
  pages={1082--1098},
  year={2021},
  publisher={IEEE}
}

@misc{moore2025rosflight20leanros,
      title={ROSflight 2.0: Lean ROS 2-Based Autopilot for Unmanned Aerial Vehicles}, 
      author={Jacob Moore and Phil Tokumaru and Ian Reid and Brandon Sutherland and Joseph Ritchie and Gabe Snow and Tim McLain},
      year={2025},
      eprint={2510.00995},
      archivePrefix={arXiv},
      primaryClass={cs.RO},
      url={https://arxiv.org/abs/2510.00995}, 
}

@article{perlin1985image,
  title={An image synthesizer},
  author={Perlin, Ken},
  journal={ACM Siggraph Computer Graphics},
  volume={19},
  number={3},
  pages={287--296},
  year={1985},
  publisher={ACM New York, NY, USA}
}

\end{document}